\documentclass[10pt,twocolumn,letterpaper]{article}

\usepackage{cvpr}
\usepackage{times}
\usepackage{epsfig}
\usepackage{epstopdf}
\usepackage{float}
\usepackage{graphicx}
\usepackage{amsmath}
\usepackage{amsthm}
\usepackage{amssymb}
\usepackage{booktabs}
\usepackage[lined,boxed,commentsnumbered,ruled]{algorithm2e}

\newcommand{\R}[1]{\mathbb{R}^{#1}}
\newcommand{\RR}[2]{\mathbb{R}^{#1 \times #2}}

\newcommand{\conv}[1]{\mbox{conv}\left(#1\right)}
\newcommand{\diag}{\mbox{diag}}
\newcommand{\proj}{\mathcal{P}}

\newcommand{\prox}{\mbox{prox}}
\newcommand{\st}{\mbox{s.t.~}}

\newcommand{\refLemma}[1]{Lemma~\ref{#1}}
\newcommand{\refEq}[1]{(\ref{#1})}
\newcommand{\refFig}[1]{Figure~\ref{#1}}
\newcommand{\refProp}[1]{Proposition~\ref{#1}}

\newcommand{\refSec}[1]{Section~\ref{#1}}

\newcommand{\refTab}[1]{Table~\ref{#1}}

\newtheorem{lemma}{Lemma}
\newtheorem{proposition}{Proposition}

\def\bfc{{\boldsymbol{c}}}

\def\bfB{{\boldsymbol{B}}}
\def\bfC{{\boldsymbol{C}}}

\def\bfI{{\boldsymbol{I}}}

\def\bfM{{\boldsymbol{M}}}

\def\bfQ{{\boldsymbol{Q}}}
\def\bfR{{\boldsymbol{R}}}
\def\bfS{{\boldsymbol{S}}}
\def\bfT{{\boldsymbol{T}}}
\def\bfU{{\boldsymbol{U}}}
\def\bfV{{\boldsymbol{V}}}
\def\bfW{{\boldsymbol{W}}}
\def\bfX{{\boldsymbol{X}}}
\def\bfY{{\boldsymbol{Y}}}
\def\bfZ{{\boldsymbol{Z}}}

\def \bfsigma {\boldsymbol{\sigma}}

\def\bfone{{\boldsymbol{1}}}
\def\half{\frac{1}{2}~}

\usepackage[pagebackref=true,breaklinks=true,letterpaper=true,colorlinks,bookmarks=false]{hyperref}

\cvprfinalcopy 

\ifcvprfinal\fi

\begin{document}

\title{3D Shape Estimation from 2D Landmarks: A Convex Relaxation Approach}
\author{Xiaowei Zhou$^\dag$, Spyridon Leonardos$^\dag$, Xiaoyan Hu$^{\ddag\dag}$, Kostas Daniilidis$^\dag$ \\[0ex]
$^\dag$ University of Pennsylvania \hspace{2em} $^\ddag$ Beijing Normal University\\
\small{\{xiaowz,spyridon,kostas\}@cis.upenn.edu \hspace{6em} huxy@bnu.edu.cn} \hspace{4em}
}

\maketitle

\begin{abstract}
We investigate the problem of estimating the 3D shape of an object, given a set of 2D landmarks in a single image. To alleviate the reconstruction ambiguity, a widely-used approach is to confine the unknown 3D shape within a shape space built upon existing shapes. While this approach has proven to be successful in various applications, a challenging issue remains, i.e., the joint estimation of shape parameters and camera-pose parameters requires to solve a nonconvex optimization problem. The existing methods often adopt an alternating minimization scheme to locally update the parameters, and consequently the solution is sensitive to initialization. In this paper, we propose a convex formulation to address this problem and develop an efficient algorithm to solve the proposed convex program. We demonstrate the exact recovery property of the proposed method, its merits compared to alternative methods, and the applicability in human pose and car shape estimation.
\end{abstract}

\maketitle

\section{Introduction}

Recognizing 3D objects from 2D images is a central problem in computer vision. In recent years, there has been an emerging trend towards analyzing 3D geometry of objects including shapes and poses instead of merely providing bounding boxes \cite{xiang2012estimating,lim2013parsing,aubry2014seeing,mori2006recovering,wei2009modeling,sigal2012loose}. The 3D geometric reasoning can not only provide richer information about the scene for subsequent high-level tasks, but also improve the performance of object detection. \cite{fidler20123d,pepik20123d2pm,andriluka2010monocular,simo2013joint}.

Estimating the 3D geometry of an object from a single view is an ill-posed problem. But it is a possible task for a human observer, since human can leverage visual memory of object shapes. Inspired by this idea, more and more efforts have been made towards 3D model-based analysis leveraging the increasing availability of online 3D model databases. To address intra-class variability or nonrigid deformation, many recent works, e.g., \cite{hejrati2012analyzing},\cite{ramakrishna2012reconstructing},\cite{zia2013detailed},\cite{wang2014robust}, have adopted a shape-space approach originated from the ``active shape model" \cite{cootes1995active}, where each shape is defined by a set of ordered landmarks and the shape to be estimated is assumed to be a linear combination of predefined basis shapes. For estimation, the 3D shape model is fitted to the landmarks annotated or detected in images. In this way, the problem turns into a 3D-to-2D shape fitting problem, where the shape parameters (weights of the linear combination) and the pose parameters (viewpoint) have to be estimated simultaneously.

While this approach has achieved promising results in various applications, the model inference is still a challenging problem, since the subproblems of shape and pose estimation are coupled: the pose needs to be known to deform the 3D model to match 2D landmarks, while the exact 3D model is required to estimate the pose. The joint estimation of shape and pose parameters usually results in a nonconvex optimization problem, and the orthogonality constraint on the pose parameters makes the problem more complicated. The previous methods often adopted an alternating scheme to alternately update the shape and pose parameters until convergence. Therefore, the algorithms were sensitive to initialization and may get stuck at locally-optimal solutions. As mentioned in many works, e.g., \cite{ramakrishna2012reconstructing},\cite{hejrati2012analyzing}, most of the failed cases were attributed to bad initialization. Some heuristics have been used to relieve this issue, such as trying multiple initializations \cite{wang2014robust} or using a viewpoint-aware detector for pose initialization \cite{zia2013detailed}. But there is still no guarantee for global optimality.

In this paper, we propose a convex relaxation approach to addressing the aforementioned issue:
\begin{itemize}
\item[1.] We use an augmented shape-space model, where a shape is represented as a linear combination of rotatable basis shapes. This model can give a linear representation of both intrinsic shape deformation and extrinsic viewpoint changes.
\item[2.] We use the convex relaxation of the orthogonality constraint to convert the entire problem into a spectral-norm regularized linear inverse problem, which is a convex program.
\item[3.] We develop an efficient algorithm to globally solve the proposed convex program.
\end{itemize}

The remainder of this paper is organized as follows. We first give a brief introduction to related work in \refSec{sec:related}. Then, we explain the formulation in \refSec{sec:formulation} and provide the algorithm in \refSec{sec:optimization}. Next, we experimentally demonstrate the merits and applicability of the proposed method in \refSec{sec:experiments}. Finally, we conclude the paper with some discussions in \refSec{sec:discussion}.

\section{Related Work}\label{sec:related}

The most related work includes the papers that solve shape estimation by fitting a shape-space model to 2D landmarks. This approach has been successfully applied to reconstruction of a variety of objects including human poses \cite{ramakrishna2012reconstructing,wang2014robust,fan2014pose,zhou2014sptio}, cars \cite{li2011robustly,hejrati2012analyzing,zia2013detailed,lin2014jointly}, faces \cite{blanz2003face,gu20063d,cao20133d}, to name a few. Following are a few recent examples.

Ramakrishna et al. \cite{ramakrishna2012reconstructing} proposed a sparse representation based approach to reconstructing 3D human pose from annotated landmarks in a still image. Wang et al. \cite{wang2014robust} adopted a 2D human pose detector \cite{yang2011articulated} to automatically locate the joints and used a robust estimator to handle inaccurate joint locations. Fan et al. \cite{fan2014pose} proposed to improve the performance of \cite{ramakrishna2012reconstructing} by enforcing locality when building the pose dictionary. Hejrati et al. \cite{hejrati2012analyzing} used the active shape model for 3D car reconstruction and produced 2D landmarks by a variant of deformable part models \cite{felzenszwalb2008discriminatively}. Lin et al. \cite{lin2014jointly} proposed a method for joint 3D model fitting and fine-grained classification for cars. In some works, landmark locations were estimated jointly with shape fitting. For example, Zia \cite{zia2013detailed} et al. developed a probabilistic framework to simultaneously localize 2D landmarks and recovery 3D object models. Zhou et al. \cite{zhou2014sptio} formulated human pose estimation as a matching problem, where the learned spatio-temporal pose model was matched to extracted trajectories in a video.

A common component or an intermediate step in these works is the 3D model fitting to 2D landmarks. As mentioned in the introduction, the previous work usually relied on nonconvex formulations, which may be sensitive to initialization. The convex formulation proposed in this paper can potentially serve as a building block to improve the performance of the existing methods.

Our work is also closely related to nonrigid structure from motion (NRSfM), where a deformable shape is recovered from multi-frame 2D-2D correspondences. The low-rank shape-space model has been frequently used in NRSfM, but the basis shapes are unknown. The joint estimation of shape/pose variables and basis shapes is typically solved via matrix factorization followed by metric rectification \cite{bregler2000recovering,xiao2006closed}. In some recent works, iterative algorithms were employed for better precision \cite{paladini2012optimal,del2012bilinear} or sequential processing \cite{agudo2014good}, and the problem studied in this paper is analogous to the step of fixing basis shapes and updating the remaining variables in those iterative methods for NRSfM.

\section{Formulation}\label{sec:formulation}

\subsection{Problem Statement}

The problem studied in this paper can be described by the following linear system:
\begin{align}\label{eq:basic}
\bfW = \Pi\bfS,
\end{align}
where $\bfS\in\RR{3}{p}$ denotes the unknown 3D shape, which is represented by 3D locations of $p$ points. $\bfW\in\RR{2}{p}$ denotes their projections in a 2D image. $\Pi$ is the camera calibration matrix. To simplify the problem, the weak-perspective camera model is usually used, which is a good approximation when the object depth is much smaller than the distance from the camera. With this assumption, the calibration matrix has the following simple form:
\begin{align}\label{eq:calibration}
\Pi = \left(
          \begin{array}{ccc}
            \alpha & 0 & 0 \\
            0 & \alpha & 0 \\
          \end{array}
        \right),
\end{align}
where $\alpha$ is a scalar depending on the focal length and the distance to the object.

There are always more variables than equations in \refEq{eq:basic}. To make the problem well-posed, a widely-used assumption is that the unknown shape can be represented as a linear combination of predefined basis shapes, which is originated from the active shape model \cite{cootes1995active}:
\begin{align}\label{eq:shapespace}
    \bfS = \sum_{i=1}^{k} c_i\bfB_i,
\end{align}
where $\bfB_i\in\RR{3}{p}$ for $i\in[1,k]$ represents a basis shape learned from training data, while $c_i$ denotes the weight of each basis shape. In this way, the reconstruction problem is turned into a problem of estimating several coefficients by fitting the model \refEq{eq:shapespace} to the landmarks in an image, which greatly reduces the number of unknowns.

Since the basis shapes are predefined, the relative rotation and translation between the camera frame and the frame defining the basis shapes need to be taken into account, and the 3D-2D projection is depicted by:
\begin{align}\label{eq:2d3dcorresp}
\bfW = \Pi\left(\bfR\sum_{i=1}^{k} c_i\bfB_i + \bfT\bfone^T\right),
\end{align}
where $\bfR\in\RR{3}{3}$ and $\bfT\in\R{3}$ correspond to the rotation matrix and the translation vector, respectively. $\bfR$ should be in the special orthogonal group
\begin{align}
SO(3) = \{\bfR\in\RR{3}{3}|\bfR^T\bfR=\bfI_3,\det{\bfR}=1\}.
\end{align}

Equation \refEq{eq:2d3dcorresp} can be further simplified as
\begin{align}\label{eq:bilinear}
\bfW = \bar{\bfR}\sum_{i=1}^{k} c_i\bfB_i,
\end{align}
where $\bar{\bfR}\in\RR{2}{3}$ denotes the first two rows of the rotation matrix, and the translation $\bfT$ has been eliminated by centralizing the data, i.e. subtracting each row of $\bfW$ and $\bfB$ by its mean. Note that the scalar $\alpha$ in the calibration matrix has been absorbed into $c_1,\cdots,c_k$.

In the active shape model, the number of basis shapes is set to be small, which assumes that the unknown shape lies in a low-dimensional linear space. In many recent works \cite{ramakrishna2012reconstructing,zhang2011sparse,zhu2010model,zhu2014complex}, it has been shown that the low-dimensional linear space is insufficient to model complex shape variation, e.g., human poses, and a promising approach is using an over-complete dictionary and representing an unknown shape as a sparse combination of atoms in the dictionary. Such a sparse representation implicitly encodes the assumption that the unknown shape should lie in a union of subspaces that approximates a nonlinear shape manifold.

Based on the sparse representation of shapes, the following optimization problem is often considered to estimate an unknown shape:
\begin{align}\label{eq:originalnoisy}
    \min_{\bfc,\bar{\bfR}} ~~ & \half \left\| \bfW - \bar{\bfR}\sum_{i=1}^{k} c_i\bfB_i \right\|_F^2 + \lambda \|\bfc\|_1, \nonumber \\
    \st ~~ & \bar{\bfR}\bar{\bfR}^T = \bfI_2,
\end{align}
where $\bfc=[c_1,\cdots,c_k]^T$ and $\|\bfc\|_1$ represents the $\ell_1$ norm of $\bfc$, which is the convex surrogate of the cardinality. $\|\cdot\|_F$ denotes the Frobenius norm of a matrix. The cost function terms in \refEq{eq:originalnoisy} correspond to the reprojection error and the sparsity of representation, respectively.

The optimization in \refEq{eq:originalnoisy} is nonconvex and there is an orthogonality constraint. A commonly-used strategy is the alternating minimization scheme, in which two steps are alternated: fixing $\bar{\bfR}$ and updating $\bfc$ by solving the $\ell_1$ minimization problem; and fixing $\bfc$ and updating $\bar{\bfR}$ using certain rotation representations such as the quaternions, the exponential map or a manifold representation. Note that the Procrustes method cannot be directly applied here since $\bar{\bfR}\in\RR{2}{3}$ is not a full rotation matrix and generally no closed-form solution exists \cite{edelman1998geometry}. Consequently, the whole algorithm may get stuck at local minima far away from the globally-optimal solution.

\subsection{Proposed Model}

We propose to use the following shape-space model:
\begin{align}\label{eq:new3dmodel}
\bfS = \sum_{i=1}^{k} c_i\bfR_i\bfB_i,
\end{align}
in which there is a rotation for each basis shape. The model in \refEq{eq:new3dmodel} implicitly accounts for the viewpoint variability and the projected 2D model is
\begin{align}\label{eq:new2dmodel}
\bfW = \Pi\sum_{i=1}^{k} c_i\bfR_i\bfB_i = \sum_{i=1}^{k} \bfM_i\bfB_i,
\end{align}
where $\bfM_i\in\RR{2}{3}$ is the product of $c_i$ and the first two rows of $\bfR_i$, which satisfies
\begin{align}\label{eq:orthogonality}
    \bfM_i\bfM_i^T = c_i^2\bfI_2.
\end{align}

The motivation of using the models in \refEq{eq:new3dmodel} and \refEq{eq:new2dmodel} is to achieve a linear representation of shape variability in 2D, such that we can get rid of the bilinear form in \refEq{eq:bilinear}, which is a necessary step towards a convex formulation.

The model in \refEq{eq:new2dmodel} is equivalent to the affine-shape model in existing literature \cite{blake2000active,xiao2004real}, which uses an augmented linear space to represent the shape variation in 2D caused by both intrinsic shape deformation and extrinsic viewpoint changes. This representation also appears in most NRSfM literature \cite{bregler2000recovering,paladini2012optimal}. As mentioned in \cite{xiao2004real}, the augmented linear space can represent any 2D shape produced by the 3D shape model projected into the image plane, but the increase of degree of freedom may result in invalid shapes. In this work, we try to reduce the possibility of invalid cases by enforcing the orthogonality constraint on $\bfM_i$s and the sparsity constraint on the number of activated basis shapes. We will show that these constraints can be conveniently imposed by minimizing a convex regularizer.

Next, we will consider to replace the orthogonality constraint in \refEq{eq:orthogonality} by its convex counterpart. The following lemma has been proven in literature \cite[Section 3.4]{journee2010generalized}:

\begin{lemma}\label{lemma1}
The convex hull of the Stiefel manifold $\mathcal{Q}=\left\{\bfX\in\RR{m}{n} | ~\bfX^T\bfX = \bfI_n\right\}$ equals the unit spectral-norm ball $\conv{\mathcal{Q}}=\left\{\bfX\in\RR{m}{n} | ~~\|\bfX\|_2 \leq 1 \right\}$. $\|\bfX\|_2$ denotes the spectral norm (a.k.a. the induced 2-norm) of a matrix $\bfX$, which is defined as the largest singular value of $\bfX$.
\end{lemma}
Based on \refLemma{lemma1}, we have the following proposition:
\begin{proposition}
Given a scalar $s$, the convex hull of $\mathcal{S}=\left\{\bfY\in\RR{m}{n} | ~\bfY^T\bfY = s^2\bfI_n\right\}$ equals the spectral-norm ball with a radius of $|s|$: $\conv{\mathcal{S}}=\left\{\bfY\in\RR{m}{n} | ~~\|\bfY\|_2 \leq |s| \right\}$.
\end{proposition}
The proof is straightforward since there is a linear mapping between $\mathcal{S}$ and $\mathcal{Q}$ by $\bfY=s\bfX$.
\vspace{0.5em}

Consequently, the tightest convex relaxation to the constraint in \refEq{eq:orthogonality} is given by $\|\bfM_i\|_2 \leq |c_i|$.


Finally, with the modified shape model, the relaxed orthogonality constraint and the assumption of sparse representation, we propose to minimize the $\ell_1$-norm of the coefficient vector for shape recovery under noiseless cases:
\begin{align}
    \min_{c_1,\cdots,c_k,\bfM_1,\cdots,\bfM_k}~ & \sum_{i=1}^{k}|c_i|, \nonumber \\
    \st ~~~~~~~~~~ & \bfW = \sum_{i=1}^{k} \bfM_i\bfB_i, \nonumber \\
    & \|\bfM_i\|_2 \leq |c_i|, ~ \forall i\in[1,k]
\end{align}
which is obviously equivalent to the following problem:
\begin{align}\label{eq:finalnoiseless}
    \min_{\bfM_1,\cdots,\bfM_k}~ & \sum_{i=1}^{k}\|\bfM_i\|_2, \nonumber \\
    \st ~~~~ & \bfW = \sum_{i=1}^{k} \bfM_i\bfB_i.
\end{align}
The formulation in \refEq{eq:finalnoiseless} is a linear inverse problem, where we estimate a set of orthogonal matrices by minimizing their spectral norms. Interestingly, the conditions for exact recovery using such a convex program has been theoretically analyzed in \cite{chandrasekaran2012convex}. We will provide numerical results to demonstrate the exact recovery property in \refSec{sec:simulation}.

Considering noise in real applications, we can solve:
\begin{align}\label{eq:finalnoisy}
    \min_{\bfM_1,\cdots,\bfM_k}~ & \half \left\| \bfW - \sum_{i=1}^{k} \bfM_i\bfB_i \right\|_{F}^2 + \lambda \sum_{i=1}^{k}\|\bfM_i\|_2.
\end{align}
The problem \refEq{eq:finalnoisy} is our final formulation. It is a penalized least-squares problem. We have following remarks:
\begin{itemize}
\item[1.] The problem in \refEq{eq:finalnoisy} is convex programming, which can be solved globally. We will provide an efficient algorithm to solve it in \refSec{sec:optimization}.
\item[2.] Notice that $\|\cdot\|_2$ in the above formulations denotes the spectral norm of a matrix instead of the $\ell_2$-norm of a vector. As we will show in \refSec{sec:optimization}, minimizing the spectral norm of a matrix is equivalent to minimizing the $\ell_{\infty}$-norm of the vector of singular values, which will simultaneously shrink the norm of the matrix towards zero and enforce its singular values to be equal. Therefore, by spectral-norm minimization, we can not only minimize the number of activated basis shapes but also enforce each transformation matrix $\bfM_i$ to be orthogonal (an orthogonal matrix has equal singular values).
\item[3.] In practice, we may estimate $\bfM_i$s by only considering reprojection errors at visible landmarks, i.e., including a binary weight matrix in the first term of \refEq{eq:finalnoisy}. The missing landmarks can be hallucinated from the reconstructed shape model as their locations are known on the basis shapes.
\end{itemize}

\subsection{Reconstruction}

After solving \refEq{eq:finalnoisy}, we recover $c_i$ and $\bfR_i$ from the estimated $\bfM_i$, and reconstruct the 3D shape by \refEq{eq:new3dmodel}. Specifically, $c_i=\|\bfM_i\|_2$ and $\bar{\bfR}_i=\bfM_i/c_i$. Note that $c_i=-\|\bfM_i\|_2$ is also a feasible solution. To eliminate the ambiguity, we assume that $c_i\geq 0$ and impose this constraint when training the shape dictionary. Finally, the third row of $\bfR_i$ is recovered by the cross product of the rows in $\bar{\bfR}_i$.


\section{Optimization}\label{sec:optimization}

\subsection{Proximal operator of the spectral norm}\label{sec:proximal}

Before deriving the specific algorithm to solve \refEq{eq:finalnoisy}, we first prove the following proposition, which will serve as an important building block in our algorithm.
\begin{proposition}\label{prop:prox2norm}
The solution to the following problem
\begin{align}\label{eq:prox-2norm}
\min_{\bfX} ~ \half \|\bfY-\bfX\|_F^2 + \lambda \|\bfX\|_2
\end{align}
is given by $\bfX^*=\mathcal{D}_{\lambda}(\bfY)$, where
\begin{align}
\mathcal{D}_{\lambda}(\bfY) &= \bfU_Y~\diag\left[\bfsigma_Y - \lambda\proj_{\ell_1}(\bfsigma_Y/\lambda)\right]~\bfV_Y^T,
\end{align}
$\bfU_Y$, $\bfV_Y$ and $\bfsigma_Y$ denote the left singular vectors, right singular vectors and the singular values of $\bfY$, respectively. $\proj_{\ell_1}$ is the projection of a vector to the unit $\ell_1$-norm ball.
\end{proposition}

\begin{proof}
The problem in \refEq{eq:prox-2norm} is a proximal problem \cite{parikh2013proximal}. The proximal problem associated with a function $F$ is defined as
\begin{align}
\prox_{\lambda F}(\bfY) = \arg\min_{\bfX} \half \|\bfY-\bfX\|_F^2 + \lambda F(\bfX),
\end{align}
with the solution denoted by $\prox_{\lambda F}(\bfY)$ and named the proximal operator of $F$.

For the problem \refEq{eq:prox-2norm}, $F(\bfX)=\|\bfX\|_2=\|\bfsigma_X\|_{\infty}$, where $\|\cdot\|_{\infty}$ means the $\ell_{\infty}$ norm. It says that $F$ is a spectral function operated on singular values of a matrix. Based on the property of spectral functions \cite[Section 6.7.2]{parikh2013proximal}, we have
\begin{align}
\prox_{\lambda F}(\bfY) = \bfU_Y~\diag\left[\prox_{\lambda f}(\bfsigma_Y)\right]~\bfV_Y^T,
\end{align}
where $f$ is the $\ell_{\infty}$-norm. The proximal operator of the $\ell_{\infty}$-norm can be computed by Moreau decomposition \cite[Section 6.5]{parikh2013proximal}:
\begin{align}
\prox_{\lambda f}(\bfsigma_Y) = \bfsigma_Y - \lambda\proj_{\ell_1}(\bfsigma_Y/\lambda),
\end{align}
given that the $\ell_1$-norm is the dual norm of the $\ell_{\infty}$-norm.
\end{proof}

\subsection{Algorithms}\label{sec:alg}

We present the algorithm to solve \refEq{eq:finalnoisy}. The noiseless case \refEq{eq:finalnoiseless} can be solved similarly. Our algorithm is based on the Alternating Direction Method of Multipliers (ADMM) \cite{boyd2010distributed} and the proximal operator of the spectral norm.

We first introduce an auxiliary variable $\bfZ$ and rewrite \refEq{eq:finalnoisy} as follows
\begin{align}\label{eq:admm-noisy}
    \min_{\widetilde{\bfM},\widetilde{\bfZ}}~ & \half \left\| \bfW - \bfZ\widetilde{\bfB}  \right\|_{F}^2 + \lambda \sum_{i=1}^{k}\|\bfM_i\|_2, \nonumber \\
    \st ~& \widetilde{\bfM} = \bfZ,
\end{align}
where we concatenate $\bfM_1,\cdots,\bfM_k$ as column-triplets of $\widetilde{\bfM}$ and $\bfB_1,\cdots,\bfB_k$ as row-triplets of $\widetilde{\bfB}$.

The augmented Lagrangian of \refEq{eq:admm-noisy} is
\begin{align}
    \mathcal{L}_{\mu}\left(\widetilde{\bfM},\bfZ,\bfY\right) &= \half \left\| \bfW - \bfZ\widetilde{\bfB}  \right\|_{F}^2 + \lambda \sum_{i=1}^{k}\|\bfM_i\|_2 \nonumber \\
    & + \left<\bfY,\widetilde{\bfM}-\bfZ\right> + \frac{\mu}{2}\left\| \widetilde{\bfM} - \bfZ \right\|_F^2,
\end{align}
where $\bfY$ is the dual variable and $\mu$ is a parameter controlling the step size in optimization. Then, the ADMM alternates the following steps until convergence:
\begin{align}
    &\widetilde{\bfM}^{t+1} = \arg\min_{\widetilde{\bfM}}\mathcal{L}_{\mu}\left(\widetilde{\bfM},\bfZ^t,\bfY^t\right); \label{eq:admm1} \\
    &\bfZ^{t+1} = \arg\min_{\bfZ}\mathcal{L}_{\mu}\left(\widetilde{\bfM}^{t+1},\bfZ,\bfY^t\right); \label{eq:admm2} \\
    &\bfY^{t+1} = \bfY^{k} + \mu~\left(\widetilde{\bfM}^{t+1}-\bfZ^{t+1}\right). \label{eq:admm3}
\end{align}

For the step in \refEq{eq:admm1}, we have
\begin{align}\label{eq:admm11}
&\min_{\widetilde{\bfM}}\mathcal{L}_{\mu}\left(\widetilde{\bfM},\bfZ^t,\bfY^t\right) \nonumber \\
=&\min_{\widetilde{\bfM}} \frac{1}{2}\left\| \widetilde{\bfM} - \bfZ^t + \frac{1}{\mu}\bfY^t \right\|_F^2 + \frac{\lambda}{\mu} \sum_{i=1}^{k}\|\bfM_i\|_2 \nonumber \\
=&\min_{\bfM_1,\cdots,\bfM_k} \sum_{i=1}^{k} \left\{ \frac{1}{2}\left\| \bfM_i - \bfQ_i^t \right\|_F^2 + \frac{\lambda}{\mu}\|\bfM_i\|_2 \right\},
\end{align}
where $\bfQ_i^t$ is the $i$-th column-triplet of $\bfZ^t - \frac{1}{\mu}\bfY^t$. Therefore, we can update each $\bfM_i$ by solving a proximal problem based on \refProp{prop:prox2norm}:
\begin{align}
\bfM_i^{t+1} = \mathcal{D}_{\frac{\lambda}{\mu}}(\bfQ_i^t), ~~\forall i\in[1,k].
\end{align}

For the step in \refEq{eq:admm2}, $\mathcal{L}_{\mu}\left(\widetilde{\bfM}^{t+1},\bfZ,\bfY^t\right)$ is a quadratic form of $\bfZ$ and admits the following closed-form solution:
{\small
\begin{align}
\bfZ^{t+1} = \left( \bfW\widetilde{\bfB}^T+\mu\widetilde{\bfM}^{t+1}+\bfY^t \right) \left( \widetilde{\bfB}\widetilde{\bfB}^T+\mu\bfI \right)^{-1}.
\end{align}
}

It can be proven that the sequences of values produced by the ADMM iterations in \refEq{eq:admm1} to \refEq{eq:admm3} converge to the optimal solutions of the primal problem in \refEq{eq:admm-noisy} \cite{boyd2010distributed}, which are also the optimal solutions to the original problem in \refEq{eq:finalnoisy}.


\section{Experiments}\label{sec:experiments}

\subsection{Simulation}\label{sec:simulation}

\begin{figure}
\centering
\includegraphics[width=0.7\linewidth]{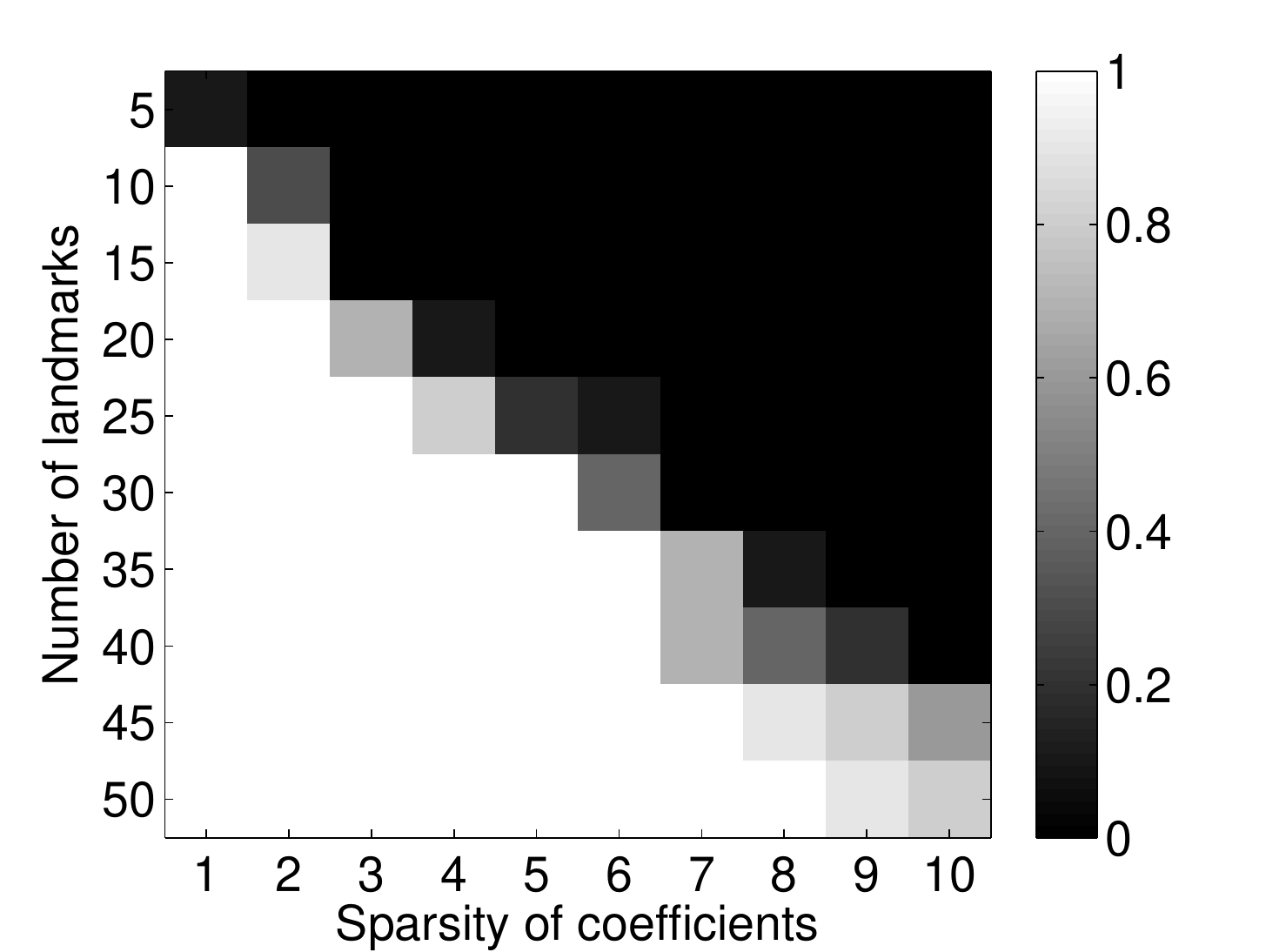}
\caption{The frequency of exact recovery on synthesized data.} \label{fig:freqexact}
\vspace{-1em}
\end{figure}

We aim to investigate whether the spectral-norm minimization in \refEq{eq:finalnoiseless} can exactly solve the ill-posed inverse problem based on the prior knowledge of sparsity and orthogonality under noiseless cases.

More specifically, we randomly simulate $k$ basis shapes $\bfB_1,\cdots,\bfB_k\in\RR{3}{p}$ ($p$ is varying, $k=50$) with entries sampled independently from the normal distribution $\mathcal{N}(0,1)$, and simulate $k$ rotation matrices $\bfR_1,\cdots,\bfR_k$ as well as coefficients $c_1,\cdots,c_k$. Only $z$ randomly-selected coefficients are nonzero with values sampled from the uniform distribution $\mathcal{U}(0,1)$. Then, $\bfM_i=c_i\bar{\bfR}_i\in\RR{2}{3}$ and $\bfW =\sum_{i=1}^{k}\bfM_i\bfB_i$. We use $\bfW$ as the input and solve \refEq{eq:finalnoiseless} to estimate $\bfM_i$s. The solution is regarded as exact if $\|\hat\bfM-\widetilde{\bfM}\|_F/\|\widetilde{\bfM}\|_F < 10^{-3}$, where we concatenate $\bfM_i$s in $\widetilde{\bfM}$, and $\hat\bfM$ is the algorithm estimate.

\refFig{fig:freqexact} reports the frequency of exact recovery with varying $p$ (number of landmarks) and $z$ (sparsity of the underlying coefficients), which is evaluated over 10 randomly-generated instances for each setting. Note that the number of unknowns ($6k$) is much larger than the number of equations ($2p$). The proposed convex program can exactly solve the problem with a frequency equal to 1 in the lower-triangular area, where the number of landmarks is sufficiently large and the coefficients are truly sparse. This demonstrates the power of convex relaxation, which has proven to be successful in various inverse problems, e.g., compressed sensing \cite{candes2008introduction} and matrix completion \cite{candes2010power}. The performance drops in more difficult cases in the upper-triangular area. This observation is analogous to the phase transition in compressive sensing, where the recovery probability also depends on the number of observations and the underlying signal sparsity \cite{donoho2009observed}.

\subsection{Applications}

\subsubsection{Human Pose Estimation}

\begin{figure}
  \centering
  \includegraphics[width=0.9\linewidth]{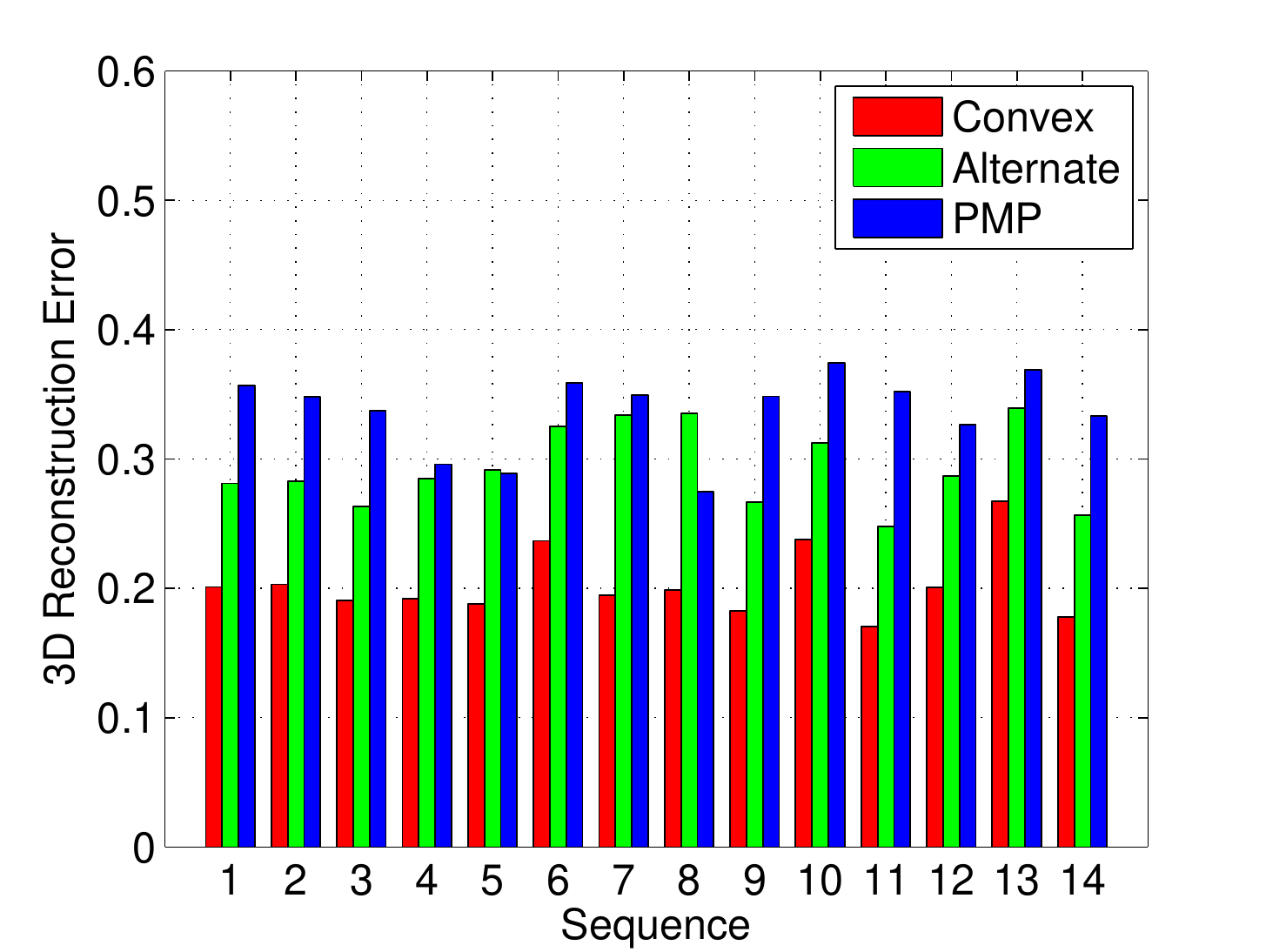}\\
  \caption{The mean reconstruction error for each sequence of Subject 15 from the MoCap dataset. Three methods are compared: ``Convex" denotes the proposed convex method; ``Alternate" means the alternating minimization method; ``PMP" represents the method proposed in \cite{ramakrishna2012reconstructing}.}\label{fig:barplot-human}
\end{figure}

\begin{table}
\renewcommand{\arraystretch}{1.2}
\centering
\begin{tabular}{llll}
\toprule
 &	Convex & Alternate & PMP \\ [0.5ex]
\hline
Subject 13 & 0.259 & 0.293 & 0.390 \\
Subject 14 & 0.258 & 0.308 & 0.393 \\
Subject 15 & 0.204 & 0.286 & 0.340 \\
\bottomrule
\end{tabular}
\vspace{1em}
\caption{The mean errors over all sequences of three subjects from the MoCap dataset.}
\label{tab:human}
\vspace{-1em}
\end{table}

The applicability of sparse shape representation for 3D human pose recovery has been thoroughly studied in previous work \cite{ramakrishna2012reconstructing,wang2014robust,fan2014pose}. In this paper, we aim to illustrate the advantage of the proposed convex program compared to the alternating minimization widely used in previous work. We carry out evaluation on the MoCap dataset \cite{mocap} and use the sequences from Subject 86 as training data and the sequences from Subject 13, 14 and 15 as testing data. All of the selected subjects are conducting a large variety of activities such as running, jumping, boxing, basketball, etc.

Since there are thousands of training shapes, using all of them as basis shapes is impractical. For our method, we solve the following problem to learn a shape dictionary:
\begin{align}\label{eq:dl}
    \min_{\bfB_1,\cdots,\bfB_k,\bfC}~ & \sum_{j=1}^{n} \half \| \bfS_j - \sum_{i=1}^{k} C_{ij}\bfB_i \|_{F}^2 + \beta \sum_{i,j}C_{ij} \nonumber \\
    \st & ~C_{ij} \geq 0, ~ \|\bfB_i\|_F \leq 1, \nonumber \\
    &~ \forall i\in[1,k], ~ j\in[1,n],
\end{align}
where $\bfB_i$s are the basis shapes to be learned, $\bfS_i$s denote the training shapes (aligned by the Procrustes method), and $C_{ij}$ represents the $i$-th coefficient for the $j$-th training shape. We initialize the dictionary by uniformly selecting $k$ shapes from the training data and locally solving \refEq{eq:dl} by alternately updating $\bfC$ and $\bfB_i$s, a strategy widely used in dictionary learning literature \cite{mairal2010online}. We use the 15 joints model as shown in \refFig{fig:vis-human} and set $k=64$

\begin{figure}
  \centering
  \includegraphics[width=0.7\linewidth]{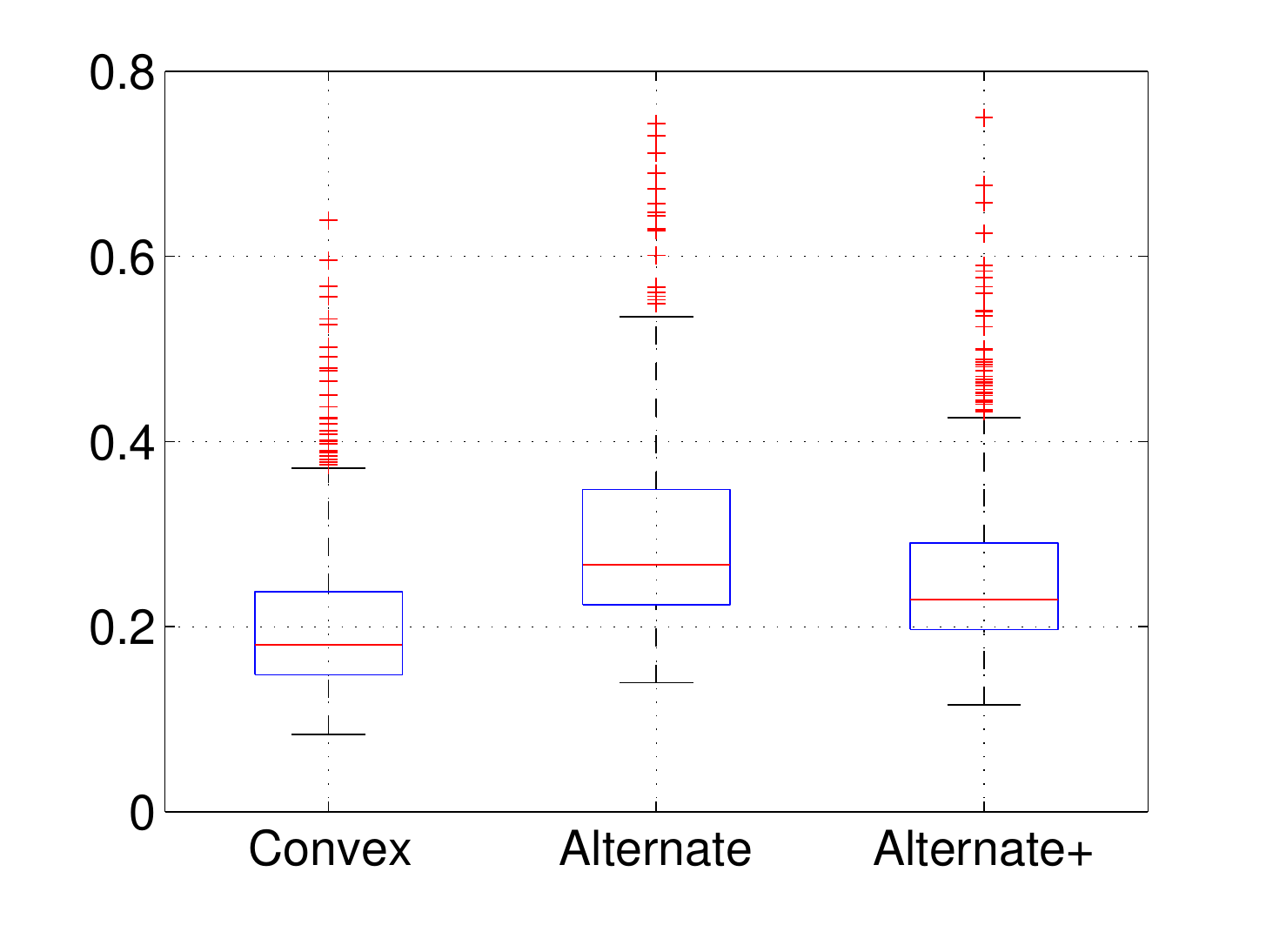}\\
  \caption{The barplots of estimation errors on the MoCap dataset (Subject 15) for the proposed method (``Convex"), the alternating minimization (``Alternate") and the alternating minimization initialized by the convex method (``Alternate+"). }\label{fig:boxplot-human}
  \vspace{-1em}
\end{figure}

We compare the proposed method to Projected Matching Pursuit (PMP) from Ramakrishna et al. \cite{ramakrishna2012reconstructing}\footnote{The code is downloaded from the authors' website \url{http://www.cs.cmu.edu/~vramakri/research.html}}. We also implement an alternating minimization method that solves the model in \refEq{eq:originalnoisy} by alternately updating the shape parameter $\bfc$ via $\ell_1$ minimization and updating the pose parameter $\bar{\bfR}$ via manifold optimization. The manifold optimization is implemented with the Manopt toolbox \cite{boumal2014manopt} to update $\bar{\bfR}$ by the trust-region algorithm over the Stiefel manifold. The alternating minimization is initialized by the mean shape of the training shapes. For both of the proposed method and the alternating minimization method, we set the regularization parameter as $\lambda=0.1$ for all sequences.

The reconstruction error is evaluated by the Euclidean distance between the reconstructed shape and the true shape up to a similarity transformation. The mean errors for the 14 testing sequences from Subject 15 are shown in \refFig{fig:barplot-human}. The subject is conducting various activities in different sequences \cite{mocap}. The proposed convex algorithm clearly outperforms the alternative methods and achieves a stable performance for all sequences. The mean error over all of the sequences for each subject is given in \refTab{tab:human}.

\begin{figure*}
  \centering
  Input \hspace{4em} Ground truth \hspace{5em} Convex \hspace{6em} Alternate \hspace{6em} PMP \\
  \includegraphics[width=0.98\linewidth]{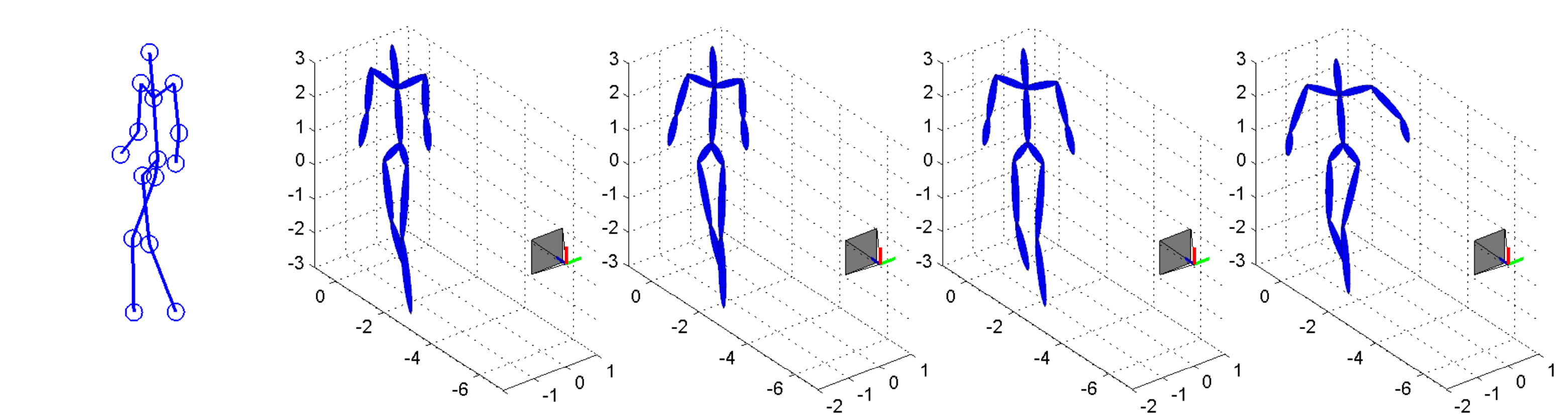}\\
  \includegraphics[width=0.98\linewidth]{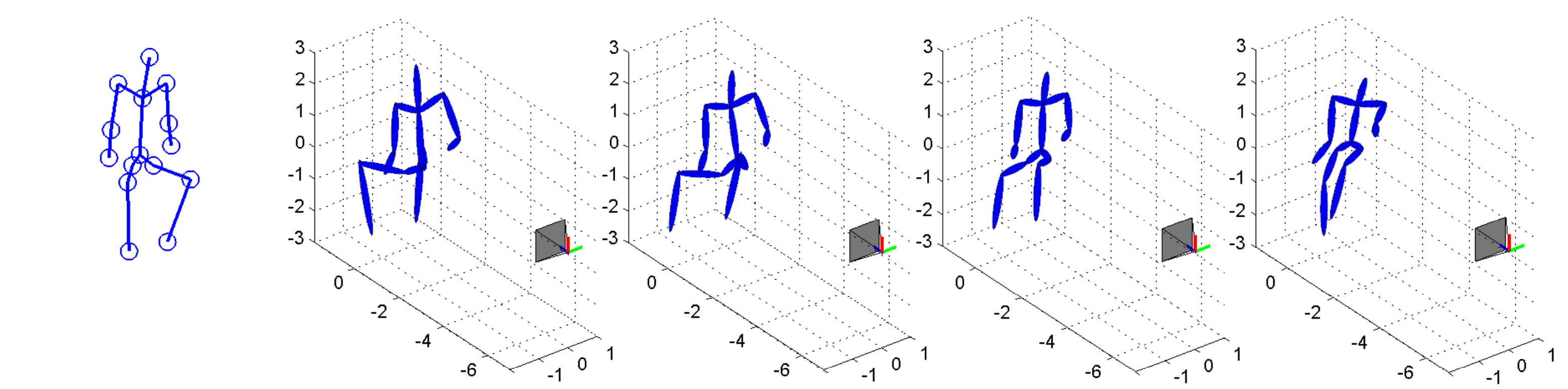}\\
  \includegraphics[width=0.98\linewidth]{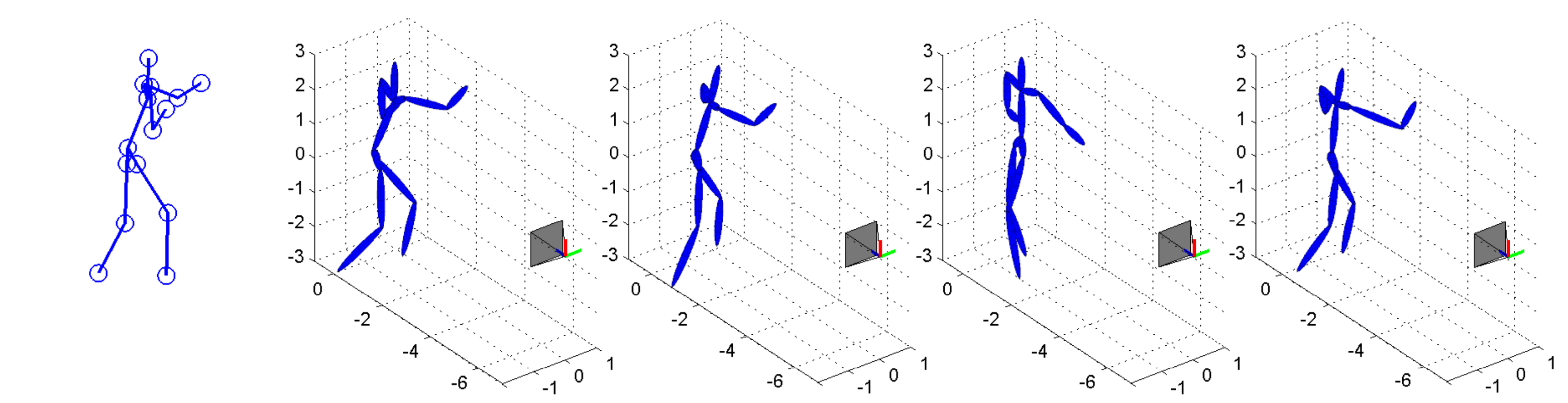}\\
  \caption{Examples of human pose estimation. The columns from left to right correspond to the input 2D landmarks, the ground-truth 3D pose, and the reconstructions from the proposed method, the alternating minimization, and the PMP method \cite{ramakrishna2012reconstructing}, respectively.}
  \label{fig:vis-human}
\end{figure*}

To verify that the alternating minimization depends on initialization, we initialize the alternating minimization with the solution of our method. The results for Subject 15 are shown in \refFig{fig:boxplot-human}. The error of the alternating minimization is apparently decreased with a smaller variance by using the better initialization. The mean objective values of alternating minimization with and without the convex initialization are 0.17 and 0.24, respectively\footnote{The objective of the convex formulation is different and therefore not compared.}. The accuracy of ``Alternate+" is worse than the convex formulation. This might be attributed to the fact that the shape model in \refEq{eq:new3dmodel} offers more degree of freedom than the original model in \refEq{eq:shapespace} to represent complex deformation of a human skeleton.

The reconstructed poses for three selected frames are visualized in \refFig{fig:vis-human}. We can see that all methods perform well in the first example, where the shape (walking) is close to the mean shape (standing straight). But the accuracies of the alternative methods degrade in the other two examples, where the shape is far away from the mean shape, while our method still obtains appealing reconstructions.

\subsubsection{Car Reconstruction}


\begin{figure*}
  \centering
  \includegraphics[width=\linewidth]{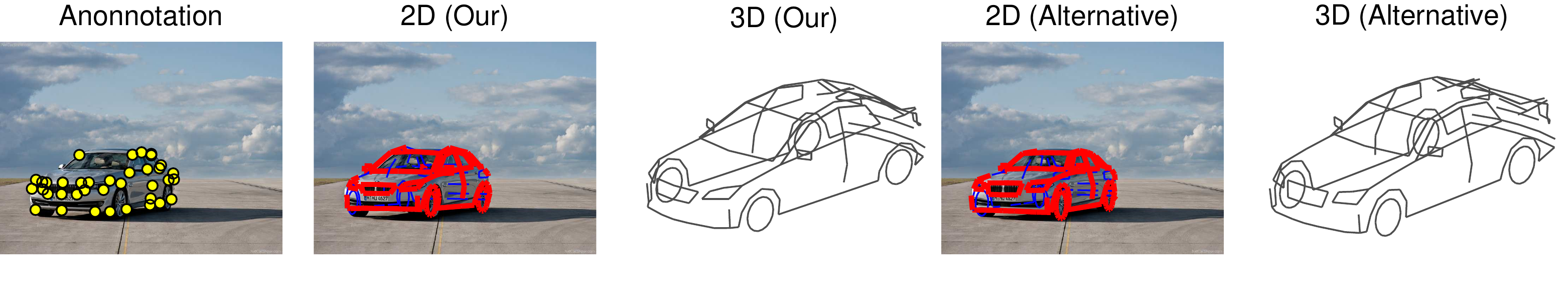}\\[-2EX]
  \includegraphics[width=\linewidth]{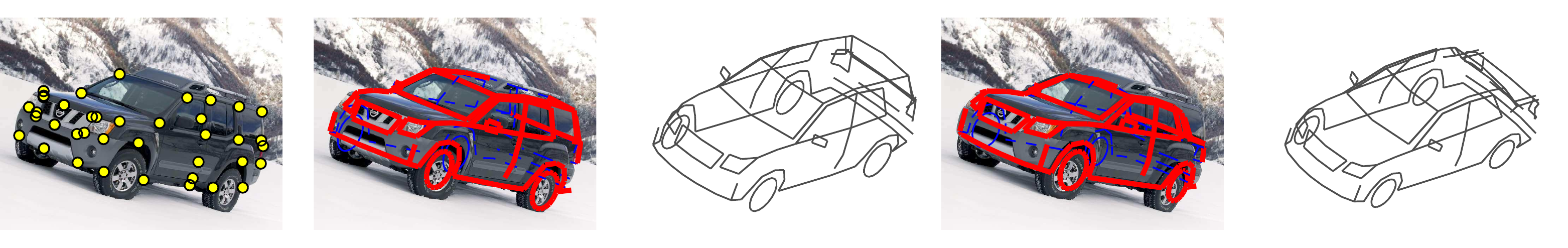}\\
  \includegraphics[width=\linewidth]{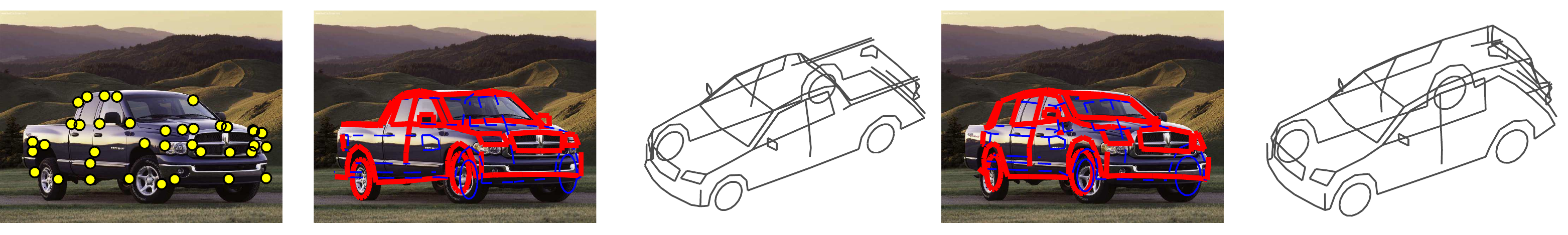}
  \caption{Examples of car reconstruction. The columns from left to right correspond to the input 2D landmarks, the 2D fitted models and 3D reconstructions of the proposed method, and the results of the alternative method \cite{lin2014jointly}, respectively. Only visible landmarks ($\sim40$ per image) are used for shape fitting. The 3D models are visualized in novel views. The car models from top to bottom are the BMW 5 Series 2011 (sedan), the Nissan Xterra 2005 (SUV) and the Dodge Ram 2003 (pick-up truck), respectively.} \label{fig:vis-car}
\end{figure*}

We demonstrate the applicability of the proposed method for 3D car shape estimation using the recently-published Fine-Grained 3D Car dataset \cite{lin2014jointly}, which provides images of cars, 2D landmark annotations and corresponding 3D models. We concatenate the 3D models of 15 cars as the shape dictionary and try to reconstruct the 3D models of other cars from the visible landmarks annotated in the images ($\sim$40 points per image). The 3D models provided in the dataset were reconstructed by the authors instead of true CAD models. Therefore, we only show some qualitative results. As illustrated in \refFig{fig:vis-car}, our method can successfully reconstruct the models of various classes such as sedan, SUV and pick-up truck. For comparison, we also show the results of an alternative method proposed in the original paper \cite{lin2014jointly}, which uses the perspective camera model and nonlinear optimization. The alternative method initialized by the mean shape performs well in the sedan example but relatively poor in the SUV and truck examples, where the models deviate far away from the mean shape. Similar results were reported in the original paper \cite{lin2014jointly} and the authors proposed to use the class-specific mean shape for better initialization. Instead, our method can achieve appealing results with arbitrary initialization.

\subsection{Computational Time}

Our algorithm is implemented in MATLAB and tested on a desktop with a Intel i7 3.4GHz CPU and 8G RAM. In our experiments, the ADMM algorithm generally converges within 500 iterations to reach a stopping criterion of $10^{-4}$. In the experiments of human pose estimation, for example, the computational time of our algorithm is 0.33s per frame, while those of the alternating minimization and the PMP algorithm \cite{ramakrishna2012reconstructing} are 0.44s and 3.02s, respectively.

\section{Discussion}\label{sec:discussion}

In summary, we proposed a method for aligning a 3D shape-space model to 2D landmarks by solving a convex program, which guarantees global optimality. Intuitively, we adopted an augmented 3D shape model to achieve a linear representation of shape variability in 2D and proposed to use the spectral-norm regularization to penalize invalid cases caused by the augmentation.

The exactness of using convex relaxation for linear inverse problems with various assumptions, e.g., sparsity and orthogonality, has been theoretically analyzed in literature, e.g., \cite{chandrasekaran2012convex}. In our experiments, we observed that the estimates satisfied the original constraints in most cases, and all reported results were the outputs of the proposed algorithm without refinement. In cases where the relaxation is not tight, postprocessing steps may be employed to enforce the exactness, e.g., projecting the estimated rotation matrix to $SO(3)$ or forcing the basis shapes to share the same rotation. This might be helpful in real applications of modeling rigid objects, although we did not use them in our experiments.

In this paper, we assume that the 2D landmarks and 3D-2D correspondences are given. Our method can be naturally extended to handle large errors in landmark localization in practice. For examples, the $\ell_1$-norm can be used to replace the squared loss in \refEq{eq:finalnoisy} to make the model more robust against outliers, and the optimization can be solved by ADMM as well. Another possible solution is to use RANSAC as proposed in \cite{li2011robustly}, since the shape model can be estimated using only a portion of the landmarks. Also, there is a great potential to integrate the proposed shape model with existing landmark-localization methods to simultaneously localize 2D landmarks and recover shapes.

\hspace{1em}

\noindent\textbf{Acknowledgments}: Grateful for support through the following grants: NSF-DGE-0966142,
NSF-IIS-1317788, NSF-IIP-1439681, NSF-IIS-1426840, ARL MAST-CTA
W911NF-08-2-0004, ARL RCTA W911NF-10-2-0016, and ONR N000141310778. Xiaoyan Hu was supported by NSFC (No.61103086 and 61170186).


\newpage
\bibliographystyle{ieee}
\footnotesize
\bibliography{mybib-abbr}

\end{document}